\documentclass{article}

     \usepackage[preprint]{neurips_2019}

\usepackage[utf8]{inputenc} 
\usepackage[T1]{fontenc}    
\usepackage{hyperref}       
\usepackage{url}            
\usepackage{booktabs}       
\usepackage{amsfonts}       
\usepackage{nicefrac}       
\usepackage{microtype}      

\usepackage{amsmath}
\usepackage{amsthm}
\usepackage[toc,page]{appendix}

\newtheorem{thm}{Theorem}
\newtheorem*{thm*}{Theorem}

\newtheorem{cor}{Corollary}[thm]
\newtheorem{lem}[thm]{Lemma}

\theoremstyle{definition}
\newtheorem{df}{Definition}

\theoremstyle{remark}
\newtheorem{ex}{Example}

\title{Neural Networks on Groups}

\author{
  Stella Rose Biderman\thanks{Also the College of Computing, Georgia Institute of Technology}\\
  Booz Allen Hamilton \\
  McLean, Virginia, USA\\
  biderman\_stella@bah.com
}

\begin{document}

\maketitle

\begin{abstract}
    Although neural networks traditionally are typically used to approximate functions defined over $\mathbb{R}^n$, the successes of graph neural networks, point-cloud neural networks, and manifold deep learning among other methods have demonstrated the clear value of leveraging neural networks to approximate functions defined over more general spaces. The theory of neural networks has not kept up however, and the relevant theoretical results (when they exist at all) have been proven on a case-by-case basis without a general theory or connection to classical work. The process of deriving new theoretical backing for each new type of network has become a bottleneck to understanding and validating new approaches.
    
    In this paper we extend the definition of neural networks to general topological groups and prove that neural networks with a single hidden layer and a bounded non-constant activation function can approximate any $L^p$ function defined over any locally compact Abelian group. This framework and universal approximation theorem encompass all of the aforementioned contexts. We also derive important corollaries and extensions with minor modification, including the case for approximating continuous functions on a compact subset, neural networks with ReLU activation functions on a linearly bi-ordered group, and neural networks with affine transformations on a vector space. Our work obtains as special cases the recent theorems of \citet{qi}, \citet{sennai}, \citet{keriven}, and \citet{maron}.
\end{abstract}

\section{Introduction}

Although neural networks are most commonly implemented as functions with real inputs and outputs, there is no particular theoretical reason to restrict oneself to that context. More generally, one can say that a neural network is a way to approximate a function from some set, $A$, to another set, $C$, that takes inputs in $A$, applies a linear function $\phi:A\to B$ and then applies an activation function $h:B\to C$. Thus we have a composition, $h(\phi(a)) = c$ which approximates a provided function $f:A\to C$. Just like in the case of functions from $\mathbb{R}^n$ to $\mathbb{R}^k$, we can create a deep neural network by stringing together several layers. Although in principle each layer could represent a different set, in practice there are typically only a handful of different sets are involved. For example, a typical neural network with three hidden layers typically maps $\{features\}\to\mathbb{R}^n\to\mathbb{R}^k\to\mathbb{R}\to\{labels\}$

In order to make sense of some of these notions - what is a linear function from the set of cars to the set of dogs? - we need to impose some additional structure on $A, B,$ and $C$ . The primary purpose of this paper is to explore how generally a neural network can be fruitfully defined and how well techniques designed to work for functions $f:\mathbb{R}^n\to\mathbb{R}^k$ generalize to those general spaces. Our primary result is a proof that the analogue of the Universal Approximation Theorem holds in a very general context widely studied in theoretical mathematics known as a \textit{locally compact Abelian group}.

Specifically, we will prove the following:

\begin{thm*}[The Main Theorem]
    Let $G$ be a locally compact Abelian group, $X$ be a normed vector space, $\phi:G\to G$ be a measurable group automorphism, $\psi:G\to X$ be a bounded and non-constant function, and $\alpha_i$ be scalars in the field underlying $X$. Then the set of functions of the form
    
    $$F(x) = \sum_{i=1}^n\alpha_i\psi(\phi(x))$$
    
    are dense in $L^p(G, X)$.
\end{thm*}

Since groups lack the requisite structure to define linear operators, we use the weaker notion of group morphisms instead to generalize neural networks to groups. The fact that we are considering automorphisms $\phi:G\to G$ instead of the more general homomorphisms $\phi:G\to H$ is solely for ease of exposition: the theorem and proof hold in that case as well. We will explore both of these facts in further detail in \autoref{sec:discussion} and present a number of extensions and special cases in the Appendix.

\subsection{Neural Networks beyond the Real Numbers}
The idea of using neural networks to approximate functions that are more general than $f:\mathbb{R}^n\to\mathbb{R}^k$ have existed for decades, although they have recently become more widespread in practice due to recent methodological advances. Perhaps the most well-known example is invariant and equivariant neural networks, which consider functions over $\mathbb{R}^n$ that are invariant or equivariant to the action of some group $G\leq S_n$ which acts as a permutation on the indices.

The classical example of an equivariant neural network is the Convolusional Neural Network (CNN)  (\citet{cnn}), where layers are restricted to be translation equivariant. Other highly successful techniques based on this paradigm include point-cloud and set networks (\citet{qi}; \citet{deep-sets}), which are invariant under all permutations, and Graph Neural Networks (\citet{gnn}, \citet{gcnn}) which are invariant under the symmetry group of a particular graph. Although these neural networks are frequently discussed in relation to $\mathbb{R}^n$, their permutation invariance and/or equivariance changes the domain in a fashion that classical universal approximation theorems can no longer be applied \cite{maron,keriven}. Instead, these types of networks have been proven to be universal approximators using highly specific techniques that do not generalize to other contexts.

Similar work on continuous data has given rise to the paradigm of deep manifold learning, which leverages the geometric structure of non-Euclidean manifolds to define neural networks analogous to Euclidean CNNs. There are a wide variety of types of manifold deep learning techniques including Spectral CNNs, correlation kernals, mixture model networks, wavelet neural networks, and monotonic chains in manifolds (\citet{manifold-chains}). As far as we are aware, no deep manifold learning techniques were known to be universal approximators prior to this work.

Finally, there has been work on using neural networks that leverage internal representations in more abstract spaces whether or not they use real-valued input data. Neural networks over the complex numbers and the quaternions have shown improved convergence with fewer parameters compared to neural networks over real numbers (\citet{deep-complex}, \citet{deep-quaternion}), neural networks over the p-adics have been used for clustering and classification (\citet{p-adic}, \citet{p-adic:cluster}), and convolutional Fourier neural networks for approximating functions in the dual space of $L^p(\mathbb{R}$) can be more efficient to train than traditional CNNs (\citet{fourier}). Prior to this work, none of these neural networks types were known to be universal approximators.

\subsection{Previous Universal Approximation Theorems}

Universal Approximation Theorems of the form ``neural networks are universal approximators of functions in space $S$ with suitable activation functions'' are known for a number of spaces. The original work on the topic applied to approximating functions on $\mathbb{R}^n$(\cite{cybenko}, \cite{hornik-et-al}, \citet{hornik}), but recent work has extended it to point-cloud networks and set networks (\cite{qi}), $S_n$-invariant networks (\cite{sennai}), $G$-invariant neural networks (\cite{maron}), and Graph-equivariant neural networks (\cite{keriven}).

It is noteworthy that the number of hidden layers considered varies across these works: the classical work on $\mathbb{R}^n$ and \citet{qi} have one hidden layer, while \citet{sennai} has two and \citet{keriven} and \citet{maron} do not put any bound on the network depth. Therefore although \citet{maron} covers the most general type of function, it does not wholly generalize the existing literature. Our result requires only one hidden layer and so does generalize all of these results. Additionally our approach generalizes both equivariant and invariant approaches, as both quotient spaces give rise to locally compact Abelian groups.

Other forms of a ``Universal Approximation Theorem'' have been considered in the literature, most notably the Width-Bound Universal Approximation Theorem (\citet{lu}, \citet{hanin}) which bounds the width but not the depth of the neural network. As far as we are aware, these have only been considered for approximating functions over $\mathbb{R}^n$. \citet{yarotsky} considers the more general context of locally compact groups, but has the augment the neural network architecture substantially beyond what is typically used. We do not consider theorems of this type in the current paper, but view it as a fruitful opportunity for future research

\subsection{Our Contributions}

The primary contribution of this paper is to generalize the concept of neural networks to groups and prove that a neural network with one hidden layer and a bounded activation function can approximate arbitrary $L^p$ and continuous functions defined over a locally compact group. In doing so, we establish a general and powerful framework for studying neural networks that we hope will ground future theoretical research on neural networks beyond $\mathbb{R}^n$. We obtain a number of important corollaries and extensions of our main result as straightforward consequences of our main proof, including the universality of affine layers, ReLU activation functions, and approximations of continuous functions on compact sets.

Our results generalize all previous universal approximation theorems that we are aware of in the literature, with the exception of the aforementioned Width-Bound Universal Approximation Theorems of \citet{lu} and \citet{hanin} and the work on polynomial neural networks of \citet{yarotsky}. Additionally, it applies applies to a number of neural network structures that were not previously known to be universal, including complex neural networks (\citet{deep-complex}), p-adic neural networks (\citet{p-adic}), Fourier neural networks (\citet{fourier}), and all manifold deep learning techniques we are aware of.

\section{Locally Compact Abelian Groups}

The central topic of this paper is \textit{locally compact Abelian groups} (LCA groups). Although LCA groups have been widely studied in mathematics, they rarely appear in computer science in full generality. Consequently, we will begin by presenting an overview of the properties of LCA groups and motivating why they are a natural place to build neural networks.

\subsection{What is a Locally Compact Abelian Group?}

Locally compact Abealian groups are a central topic in analysis because they provide many crucial tools from harmonic and functional analysis while at the same time preserving startling generality. The formal definition is as obvious as it is opaque: a locally compact Abelian group is a topological group such that the topology is locally compact (which we take to include the Hausdorf property) and the underlying group structure is Abelian. A detailed discussion of analysis over locally compact Abelian groups is outside the scope of this paper, and we refer interested readers to \citet{rudin:groups} and \citet{rudin:big} for the general theory. However, we will provide a list of common LCA groups to motivate interest and intuition before moving forward:

\begin{ex}
    The following are all examples of locally compact Abelian groups:
    \begin{enumerate}
        \item Any vector space, including $\mathbb{R}^n$ and $\mathbb{C}^n$, under vector addition
        \item The quaternions under addition (quaternion multiplication is non-Abelian)
        \item Any finite Abelian group under the discrete topology, in particular $\mathbb{F}^n_p$
        \item $\mathbb{Z}^n$ and $\mathbb{Q}^n$ under the discrete topology
        \item $\mathbb{Z}_p$ and $\mathbb{Q}_p$, the p-adic numbers under their usual topology
        \item Any manifold (Euclidean or non)
    \end{enumerate}
\end{ex}

As these examples show, LCA groups are wide-ranging and encompass many of the spaces commonly encountered in applied mathematics. From an applied point of view, one major motivation of developing a theory of neural networks over LCA groups is that almost the entire literature on neural networks deals with functions over LCA groups.

The primary example of objects that are \textit{not} LCA groups of relevance to neural networks are function spaces: $L^p(\mathbb{R})$ and $C(\mathbb{R})$ are not locally compact. It is important to emphasize that we can learn elements of these spaces in the current framework, the limitation occurs when we try to learn \textit{functions that map between} these spaces. While we are rarely interested in learning a function $f:L^p(\mathbb{R})\to L^p(\mathbb{R})$, it does come up. For example, our results do not apply to  neural ordinary differential equations (\citet{neural-ODE}). Recent work in differential topology has identified non-communative structures similar to manifolds known as ``non-communative manifolds'' (despite being, strictly speaking, not manifolds). Although manifold deep learning has no considered such structures as far as we are aware, future work in that direction may not be covered by the results of this paper.

\subsection{Properties of Locally Compact Abelian Groups}

Now we will introduce two classical properties of LCA groups that play an important role in our results. Both of these theorems are powerful statements that say something very similar to what we can say about $\mathbb{R}$, but tweaked for the more general context.

\begin{thm}[the Haar measure, \citet{rudin:groups}]
    Every LCA group, $G$, has a non-negative sigma finite measure, $m$, known as the Haar measure which has the property that, for every Borel set $B$, $m(g\dot B) = m(B)$. We are justified in saying "the Haar measure" as all Haar measures are equivalent up to a scalar multiple.
\end{thm}

$\mathbb{R}^n$ has a special measure known as the Lesbegue measure, which has the property of being translation-invariant. This property shows up in a variety of contexts and proofs and is foundational to the concept of $L^p(\mathbb{R})$. The Haar measure generalizes this concept to an arbitrary LCA group. It allows us to compute Fourier transforms, speak about $L^p(G)$, and is a cornerstone of functional analysis. The other cornerstone of functional analysis on $L^p(G)$ is the Pontryagin Duality Theorem:

\begin{thm}[Pontryagin Duality Theorem, \citet{rudin:groups}]\label{thm:duality}
    There is a canonical isomorphism, $G\cong \hat{\hat{G}}$, between a locally compact group and it's double dual.
\end{thm}

Much of the application of Fourier analysis on $\mathbb{R}^n$ leans on a highly unusual property of the real numbers: they are their own dual space. If $f:\mathbb{R}^n\to\mathbb{R}$, then the corresponding functional $\hat{f}$ \textit{is also a function from} $\mathbb{R}^n\to\mathbb{R}$. This is not the case for LCA groups in general and is an (almost) unique property of $\mathbb{R}^n$. The Pontryagin Duality Theorem is important because it allows us to carry out similar Fourier arguments to the ones we make about $\mathbb{R}^n$, just one step removed. The bulk of the work this theorem does for us is behind-the-scenes, in the development of the underlying theory of Fourier analysis on locally compact groups. We refer an interested reader to \citet{rudin:groups}, the authoritative work on the topic.

\section{Neural Networks on Groups}

Now we will present and prove the Universal Approximation Theorem for locally compact Abelian groups. Our proof is rather straightforward and the argument closely follows \citet{cybenko} and \citet{hornik}. The primary difference is the underlying theoretical machinery involved is substantially more advanced, relying heavily on the theory of harmonic and functional analysis.

\subsection{Preliminaries}\label{prelim}

First, we will present some definitions and lemmas. While the reader is likely familiar with similar theorems over $\mathbb{R}^n$, the phrasing here holds far more generally. 

Before we can define a group neural network, we need to generalize the notion of affine layer to an abstract group. Since $G$ only has one operation and no underlying field, the usual notion of an affine transformation cannot be defined. Instead we focus only on the requirement that $f(g+g')=f(g)+f(g')$.
\begin{df}[Group Homomorphisms]
    Let $G$ and $H$ be groups and let $f:G\to H$. We say that $f$ is a \textit{group homomorphism} if $f(g+g')=f(g)+f(g')$. 
\end{df}

In this work we will focus on the special case of $G=H$, in which case we can call $f$ an \textit{automorphism} of $G$. The case for $G\neq H$ follows from essentially the same proof and is presented in the Appendix.

Now we are ready to present the general definition of a group neural network:

\begin{df}[Group Neural Networks]
    A \textit{group neural network} is a function from a group $G$ to a normed vector space, $X$, of the form $$F(x) = \sum_{i=1}^n \alpha_i\psi(\phi(x))$$ where $\phi$ is a group homomorphism $\phi:G\to H$ and $\psi$ is an activation function $\phi:H\to X$.
\end{df}

and of that of a discriminatory functions, based on the definition given by \citet{cybenko}
\begin{df}[Discriminatory Functions]
    A function, $\psi$, is \textit{discriminatory} with respect to a measure, $\mu$, if 
    $$\int_G \psi(\phi(x))d\mu(x) = 0$$
    for all group automorphisms $\phi:G\to G$ implies that $\mu = 0$.
\end{df}

Finally, we will need a few standard lemmas from functional analysis for our proofs:

\begin{thm}[Hahn-Banach Theorem, \citet{rudin:big}]
    Let $X$ be a normed vector space and let $W$ be a linear subspace of $X$. Then for any $f\in \hat{W}$ there exists $g\in\hat{X}$ such that $||f|| = ||g||$ and $g\vert_{\hat{W}} = f$.
\end{thm}

We will additionally make use of a corollary of the Hahn-Banach Theorem:

\begin{cor}[\citet{rudin:big}]
    Let $X$ be a normed vector space and let $W$ be a proper closed subspace of $X$. For every $x\in X\setminus W$, let $\delta_x = \inf_{w\in W} ||x-w|| > 0$. Then there exists $f\in\hat{X}$ such that $||f|| = 1$, $f|_W = 0$, and $f(x) = \delta$.
\end{cor}

Quite confusingly, there are three theorems in functional analysis commonly referred to as the ``Riesz Representation Theorem.'' We are interested in the following one:

\begin{thm}[Riesz Representation Theorem, \citet{rudin:big}]
    Let $X$ be a locally compact Hausdorff space and let $f$ be a linear functional on $C_c(X)$. There exists a unique regular Borel measure $\mu$ on $X$ such that $$\sigma(f)=\int_Xf(x)d\mu(x)$$
\end{thm}

\subsection{The Universal Approximation Theorem}

The fundamental lemma in the proof of the Universal Approximation Theorem is a generalization of Theorem 1 in \citet{cybenko}:

\begin{lem}[The Fundamental Lemma]\label{lem:fundamental}
    Let $G$ be a locally compact Abelian group, $X$ a normed vector space, $\phi:G\to G$ be a measurable group automorphism, $\psi:G\to X$ be a discriminatory function, and $\alpha_i$ be scalars in the field underlying $X$. Then the set of functions of the form
    
    $$F(x) = \sum_{i=1}^n\alpha_i\psi(\phi(x))$$
    
    are dense in $L^p(G, X)$.
\end{lem}

\begin{proof}
    Let $S$ denote the set of all functions of the specified form, and note that $S$ is a linear subspace of $L^p(G, X)$. Assume that these sums are not dense in $L^p(G, X)$, so that there exists a $f\in L^p(G, X)$ such that $f$ is not in the closure of $S$. By the Hahn-Banach Theorem there exists a bounded linear functional, $L$, such that $L(\overline{S})=0$ and $L(f)\neq 0$.
    
    Applying the Reisz Representation Theorem to $L$ allows us to write $L(h)=\int h(t)d\mu$ for some regular Borel measure $\mu$ and all $t$. In particular, we can have $L(h) = \int \psi(\phi(x))d\mu = 0$, where this equality holds for all measurable $\phi$. This contradicts the assumption that $\psi$ is discriminatory, so $\overline{S}=C(G)$.
\end{proof}

Now all that remains is to show that standard results about discriminatory functions generalize nicely to LCA groups:

\begin{lem}\label{lem:hornik}
    Let $G$ be a locally compact Abelian group and let $\psi:G\to\mathbb{C}$. If $\psi$ is bounded and non-constant then it is discriminatory.
\end{lem}
\begin{proof}
    The proof of this theorem closely follows the proof of Theorem 5 in \citet{hornik}. The primary difference between the two cases arises due to the discussion surrounding Theorem \autoref{thm:duality}. Like Hornik, we will rely heavily on \citet{rudin:groups} and refer an interested reader there for the general theory.
    
    Let $\psi$ be a bounded and non-constant function and let $\sigma$ be a non-zero finite signed measure on $G$ such that $\int \psi(\phi(x))d\sigma(x) = 0$ for all $\phi$ satisfying the hypothesis. Let $\sigma_\varphi$ be the finite signed measure induced by the measurable function $\varphi:G\to G$ so that 
    $$\sigma_\varphi(B) = \sigma(\{g\in G:\varphi(x)\in B\})$$
    
    for all Borel sets $B$. Then for every bounded function $\chi$ we have
    
    $$\int \chi(\phi(x))d\sigma(x)= \int\chi(t)d\sigma_\phi(t)$$
    
    and in particular
    
    $$\int\psi(\phi(x))d\sigma(x) = \int\psi(t)d\sigma_\phi(t) = 0$$
    
    If $\phi\equiv 0$ then $\sigma_\phi$ is the zero measure. Therefore we consider the case of $\phi\not\equiv 0$.
    
    Let $w\in L^1(G)$ be a function whose Fourier transform has no zeros, $\mu$ be the Haar measure on $G$, and consider this integral with respect to the convolution of $w$ with $d\sigma_\phi$:
    
    \begin{align*}
        \int\psi(\phi(x))d(w\ast s_\phi)(x)
        &=\int\int\psi(\phi(x+y))w(y)d\mu(y)d\sigma_\phi(x)\\
        &=\int\left[\int\psi(\phi(x+y))w(y)d\mu(y)\right]d\sigma_\phi(x)\\
        &=\int\left[\int\psi(\phi(x+y)-\phi(y))w(y)d\mu(y)\right]d\sigma_\phi(x)\\
        &=\int\left[\int\psi(\phi(x))w(y)d\mu(y)\right]d\sigma_\phi(x)\\
        &=\int\left[\int\psi(\phi(x))d\sigma_\phi(x)\right]w(y)d\mu(y)\\
        &=0
    \end{align*}
    
    where the introduction of $-\phi(y)$ is justified by the invariance of the Haar measure under the group action, the cancellation of $\phi(y)$ and $-\phi(y)$ is justified by the fact that $\phi$ is a group homomorphism, and the exchange of variables is justified by Fubini's Theorem since:
    
    $$\int\int\psi(\phi(x+y))w(y)d\mu(y)d\sigma_\phi(x)\leq ||\psi||\cdot||\sigma_\phi||\cdot\sup |\psi(t)|<\infty$$
    
    where the first norm is the $L^1$-norm, the second is the total variation of $\sigma_\phi$, and the third is the usual absolute value on $\mathbb{C}$.

    By Theorem 1.3.5 in \citet{rudin:groups}, $L^1(G)$ is a closed ideal in the space of finite signed measures on $G$. Therefore $w\ast\sigma_\phi$ is absolutely continuous with respect to the Haar measure and we can let $h\in L^1(G)$ denote the Radon-Nikodym derivative of the convolution. Since the Fourier transform is multiplicative, we have $\hat{h}(0)=0$ and so $\int h(t)d\mu(t)=0$. We also have that $0 = \int\psi(\phi(x))d(w\ast s_\phi)(x) = \int\psi(\phi(t))h(t)d\mu(t)$. We can change variables to obtain $0 = \int\psi(\phi(t))h(t)d\mu(t) = \int\psi(t')h(\rho(t'))d\mu(t')$ where $\rho=\phi^{-1}$. If $\phi$ is not one-to-one then this change of variables introduces a constant multiplier to the integral, but since the integral is zero we can ignore the constant.
    
    We can conclude that $\int \psi(t')f(t')d\mu(t')=0$ for all $f$ in the closed $G$-invariant subspace generated by $h(\rho(t'))$ as $\rho$ ranges over non-zero automorphisms of $G$. By Theorem 7.1.2 in \citet{rudin:groups} that $G$-invariant subspace is an ideal in $L^1(G)$. In fact, it is exactly the ideal of functions, $f$, that satisfy $\int f(t)d\mu(t)=0$ by the first corollary of Theorem 7.2.4. To see this, note that, in the notation of that theorem, $Z(I)=\{0\}$ and membership in $I$ is equivalent to having a Fourier transform that is zero at zero. By the definition of the Fourier transform, this is precisely the functions that satisfy $\int f(t)dt=0$.
    
    We have now established that there is some function $f$ such that $\int\psi(t')f(t')d\mu(t')=0$ and $\int f(t)d\mu(t)=0$. Following the argument in \citet{hornik}, this implies that either $\psi$ is constant or $h\equiv 0$. The first is not true by assumption, but the second implies $\hat{\sigma}\equiv 0$ since $\hat{h}=\hat{w}\hat{\sigma}$ and $\hat{w}$ has no zeros by assumption. Since the Fourier transform of the zero function also is identically zero, we can conclude from Theorem \autoref{thm:duality} that $\sigma\equiv 0$. This contradicts our assumption that $\sigma$ is a non-zero measure, concluding the proof.
    
\end{proof}

Together, these two results establish our main theorem:

\begin{thm}[The Universal Approximation Theorem for Group Neural Networks]\label{thm:main}
    Let $G$ be a locally compact Abelian group, $X$ a normed vector space, $\phi:G\to G$ be a measurable group automorphism, $\psi:G\to X$ be a bounded and non-constant function, and $\alpha_i$ be scalars in the field underlying $X$. Then the set of functions of the form
    
    $$F(x) = \sum_{i=1}^n\alpha_i\psi(\phi(x))$$
    
    are dense in $L^p(G, X)$.
\end{thm}

\section{Discussion}\label{sec:discussion}
The techniques of this proof are highly general and highly instructive. Essentially the same proof can be applied to a number of corollaries and special cases of interest by examining the role of the assumptions in the proof. For the formal statement of some particularly relevant corollaries, see the appendix.

\begin{enumerate}
    \item Although this theorem is stated and proven for $\phi$ being a group automorphism of $G$, the same argument holds for $\phi:G\to H$ as long as $\phi$ is a group homomorphism and the measures on $G$ and $H$ are compatible. This shows that (under suitable assumptions) feature maps with activation functions are sufficent to approximate an arbitrary function.
    \item The fact that $\phi$ is a group automorphism $G$ is used only during change of variables $t\to t'$, and any other functional that allows for the change of variables to be carried out will work as well. In particular, when $G$ is an \textit{affine space} then we can allow $\phi$ to be an affine transformation. This allows for the bias term in the real case\footnote{$f(x)=Ax+b$ is a linear function as it is the equation of a line, but it is not a linear transformation as $f(x+y)\neq f(x)+f(x)$. It is an affine transformation.}. The same applies in any vector space, and in particular the use of matrices for the layers of a neural network over $\mathbb{R}^n$.
    \item The only time we use the fact that $\psi$ is bounded is to satisfy the assumptions of Fubini's Theorem. Tonelli's Theorem says that you can similarly change the order of integration when the integrand is measurable. Of particular interest is the fact that when $G$ is a linearly bi-ordered group under, this allows for the ReLU activation function to be used. Note that the ReLU function isn't defined for groups in general.
    \item All of the aforementioned results hold for approximating an arbitrary continuous function on a compact subgroup of $G$.
\end{enumerate}

Many of these remarks can be mixed and matched as desired. For example, combining $2$ and $3$ shows that the typical neural networks used over $\mathbb{R}^n$ are universal.

\subsection{Limitations and Future Work}\label{subsec:limit}

While our neural networks our defined for all groups, our universal approximation theorem focuses on locally compact Abelian groups. While local compactness is crucial to our framework, the role that Abelianness plays is much more subtle. The Abelianness assumption is crucial to the development of a classical theory of Fourier analysis, but recent work such as \citet{nonA-fourier1} and \citet{nonA-fourier2} has shown that representation theory has useful tools for constructing Fourier-like arguments for non-Abelian groups. Although a general theory of non-Abelian Fourier analysis does not currently exist, insights from that field may show how to extend the results of this paper to non-Abelian groups.

We can also wonder about practical matters, such as if a neural network over a general topological group can be trained and whether doing so is efficient. Since our proof, like classical work but unlike some recent work, is entirely non-constructive it is unclear how to find (in theory or in practice) neural networks that approximate a particular function well. The answer to even ``easy'' questions along these lines is non-obvious, as some basic notions such as the direction derivative do not exist for arbitrary locally compact Abelian groups. Indeed, as far as we are aware there is no general definition of a gradient for LCA groups. We leave these practical questions for future work.

\section{Conclusion}

We have presented introduced the notion of neural networks over groups and proven that they are universal approximators over locally compact Abelian groups. This establishes the most general universal approximation theorem for neural networks in the literature to date. In addition to generalizing the existing literature, our theorem proves results for a number of neural networks of theoretical and practical interest that had not been previously proven to be universal including complex neural networks (\citet{deep-complex}), p-adic neural networks (\citet{p-adic}), Fourier neural networks (\citet{fourier}), and manifold deep learning techniques. We hope that the unified framework presented in this paper serves as a starting-off point for a unified theoretical approach to neural networks research going forward.

Similarly to classical Universal Approximation Theorems but different from recent work, our proof is entirely non-constructive. The question of how to implement and train neural networks over locally compact Abelian groups is an interesting question of both theoretical and practical importance that we leave to future work.

\bibliographystyle{plainnat}
\bibliography{citations}

\begin{appendices}
\section{Formal Statement of Select Corollaries}

Probably the most important corollary is the case where $G$ is a vector space such as $\mathbb{R}^n$ or $\mathbb{C}^n$. Over vector spaces we are interested in using affine transformations rather than group homomorphisms in our neural networks, but that is not a problem for our proof technique:
\begin{cor}
    Let $V$ be a finite-dimensional vector space, $X$ be a normed vector space, $\phi:V\to V$ be an affine measurable function, $\psi:V\to X$ be a bounded and non-constant function, and $\alpha_i$ be scalars in the field underlying $X$. Then the set of functions of the form
    
    $$F(x) = \sum_{i=1}^n\alpha_i\psi(\phi(x))$$
    
    are dense in $L^p(V, X)$.
\end{cor}

in full generality, we have the following for all affine spaces:

\begin{cor}
    Let $G$ be an affine locally compact group, $X$ be a normed vector space, $\phi:G\to G$ be an affine measurable function, $\psi:G\to X$ be a bounded and non-constant function, and $\alpha_i$ be scalars in the field underlying $X$. Then the set of functions of the form
    
    $$F(x) = \sum_{i=1}^n\alpha_i\psi(\phi(x))$$
    
    are dense in $L^p(G, X)$.
\end{cor}

Next we consider the case of the ReLU activation function. Since groups are not necessarily ordered, a general locally compact group may not have a $\max$ operator. However a $\max$ operator can be defined when $G$ is linearly bi-ordered, giving rise to the following:

\begin{cor}
    Let $G$ be a linearly bi-ordered locally compact Abelian group, $X$ a normed vector space, $\phi:G\to G$ be a measurable group automorphism, $\psi:G\to X$ be the ReLU function, and $\alpha_i$ be scalars in the field underlying $X$. Then the set of functions of the form
    
    $$F(x) = \sum_{i=1}^n\alpha_i\psi(\phi(x))$$
    
    are dense in $L^p(G, X)$. In particular, this holds where $\psi$ is the ReLU or leaky ReLU function.
\end{cor}

and more generally we have

\begin{cor}
    Let $G$ be a locally compact Abelian group, $X$ a normed vector space, $\phi:G\to G$ be a measurable group automorphism, $\psi:G\to X$ be a non-negative measurable and non-constant function, and $\alpha_i$ be scalars in the field underlying $X$. Then the set of functions of the form
    
    $$F(x) = \sum_{i=1}^n\alpha_i\psi(\phi(x))$$
    
    are dense in $L^p(G, X)$.
\end{cor}

These corollaries 

Another area of interest is the case of approximating an arbitrary continuous function over compact sets
\begin{cor}
    Let $G$ be a compact subset of a locally compact Abelian group, $X$ a normed vector space, $\phi:G\to G$ be a measurable group automorphism, $\psi:G\to X$ be a bounded and non-constant function, and $\alpha_i$ be scalars in the field underlying $X$. Then the set of functions of the form
    
    $$F(x) = \sum_{i=1}^n\alpha_i\psi(\phi(x))$$
    
    are dense in $C(K, X)$ for any compact group $K\leq G$.
\end{cor}

Finally we present the case where the neural network layers map from one space to another. Although it is common in practice to consider $\phi:\mathbb{R}^n\to\mathbb{R}^k$, it is far less common to explicitly consider maps between extremely different spaces, such as $\phi:G\to\mathbb{R}^n$. Such functions (with range $\mathbb{R}^n$) are generally referred to as feature maps, and an interesting corollary of our work is that general feature maps can be considered as layers in a neural network without disrupting universiality.

\begin{cor}
    Let $G$ and $H$ be locally compact Abelian groups, $X$ be a normed vector space, $\phi:G\to H$ be a measurable group homomorphism such that the measure induced by $\phi$ and the measure on $G$ is uniformly continuous with respect to the Haar measure on $H$, $\psi:H\to X$ be a bounded and non-constant function, and $\alpha_i$ be scalars in the field underlying $X$. Then the set of functions of the form
    
    $$F(x) = \sum_{i=1}^n\alpha_i\psi(\phi(x))$$
    
    are dense in $L^p(G, X)$.
\end{cor}

\end{appendices}

\end{document}